\documentclass[11pt]{article}
\usepackage{amsfonts}
\usepackage{amssymb}
\usepackage{amstext}
\usepackage{amsmath}
\usepackage{natbib}
\bibpunct{[}{]}{,}{a}{,}{,}

\usepackage{mathtime}
\usepackage{graphicx}
\usepackage{epsfig}

\usepackage{times}
\usepackage{amsmath, amssymb}
\usepackage{amsthm}

\title{Differential Privacy with Compression}

\author{
Shuheng Zhou\\
Seminar f\"{u}r Statistik\\
ETH Z\"{u}rich \\
CH-8092 Z\"{u}rich, Switzerland\\
\texttt{zhou@stat.math.ethz.ch} \\
\ \\
Katrina Ligett\\
Computer Science Department\\
Carnegie Mellon University\\
Pittsburgh, PA 15213 \\
\texttt{katrina@cs.cmu.edu} \\
\ \\
Larry Wasserman\\
Department of Statistics\\
Carnegie Mellon University\\
Pittsburgh, PA 15213 \\
\texttt{larry@stat.cmu.edu} \\
}
\newcommand{\event}{\mathcal{E}}
\newcommand{\eps}{\epsilon}

\newcommand{\R}{{\mathbb R}}

\newcommand{\X}{\mathcal{X}}\newcommand{\vecx}{\vec{x}}

\newcommand{\D}{\mathcal{D}}
\newcommand{\Sset}{\mathcal{S}}

\newcommand{\func}[1]{\ensuremath{\mathrm{#1}}}

\newcommand{\mvec}{\func{vec}}
\def\argmin{\mathop{\rm argmin}}

\def\ip#1#2{\langle #1, #2\rangle}
\newcommand{\abs}[1]{\left\lvert#1\right\rvert}
\newcommand{\prob}[1]{\ensuremath{\mathbb P}\left[#1\right]}
\newcommand{\Law}[1]{\ensuremath{\mathcal{L}}\left(#1\right)}
\newcommand{\probb}[2]{\ensuremath{\mathbb P}_{#1}\left[#2\right]}

\newcommand{\norm}[1]{\left\lVert#1\right\rVert}
\newcommand{\inv}[1]{\frac{1}{#1}}
\newcommand{\expf}[1]{\exp\left\{#1\right\}}
\newcommand{\fnorm}[1]{\lVert#1\rVert_F}
\newcommand{\be}{\begin{equation}}
\newcommand{\ee}{\end{equation}}
\newcommand{\bea}{\begin{eqnarray}}
\newcommand{\eea}{\end{eqnarray}}
\newcommand{\twonorm}[1]{\left\lVert#1\right\rVert_2}

\def\supp{\mathop{\text{supp}\kern.2ex}}
\def\argmin{\mathop{\text{arg\,min}\kern.2ex}}

\def\qed{\hskip1pt $\;\;\scriptstyle\Box$}
\newenvironment{proofof}[1]{\vskip2pt{\it Proof}{ of #1}.\hskip10pt}{\qed\vskip5pt}
\newtheorem{thm}{Theorem}[section]
\newtheorem{lmma}[thm]{Lemma}
\newtheorem{propn}[thm]{Proposition}
\newtheorem{define}[thm]{Definition}
\newtheorem{exm}[thm]{Example}
\newtheorem{remrk}[thm]{Remark}
\newtheorem{clm}[thm]{Claim}
\newtheorem{coroll}[thm]{Corollary}
\newtheorem{proc}[thm]{Procedure}
\newenvironment{proposition}{\begin{propn}\hskip-6pt{\bf }\enspace \sl}{\end{propn}}
\newenvironment{theorem}{\begin{thm}\hskip-6pt{\bf }\enspace \sl}{\end{thm}}
\newenvironment{lemma}{\begin{lmma}\hskip-6pt{\bf }\enspace \sl}{\end{lmma}}

\newenvironment{definition}{\begin{define}\hskip-6pt{\bf }\enspace \sl}{\end{define}}
\newenvironment{example}{\begin{exm}\hskip-6pt{\bf }\enspace \sl}{\end{exm}}
\newenvironment{corollary}{\begin{coroll}\hskip-6pt{\bf }\enspace \sl}{\end{coroll}}
\newenvironment{remark}{\begin{remrk}\hskip-6pt{\bf }\enspace \rm}{\end{remrk}}
\newenvironment{procedure}[1]{\begin{proc}\hskip-6pt{\bf }\enspace \rm}{\end{proc}}

\begin{document}

\maketitle

\begin{abstract}
\noindent\normalsize
This work studies formal utility and privacy guarantees for a simple
multiplicative database transformation, where the data are compressed by a
random linear or affine transformation, reducing the number of data records
substantially, while preserving the number of original input variables.
We provide an analysis framework inspired by a recent concept known as
\emph{differential privacy}~\cite{Dwork:06}. 
Our goal is to show that, despite the general difficulty
of achieving the differential privacy guarantee, it is possible to
publish synthetic data that are useful for a number of common 
statistical learning applications.
This includes high dimensional sparse regression~\cite{ZLW07}, 
principal component analysis (PCA), and other statistical 
measures~\cite{Liu:06} based on the covariance of the initial data.
\end{abstract}

\section{Introduction}
\label{sec:introduction}
In statistical learning, privacy is increasingly a concern whenever
large amounts of confidential data are manipulated within or published
outside an organization.  It is often important to allow researchers
to analyze data \emph{utility} without leaking information or 
compromising the \emph{privacy} of individual records. In this work, 
we demonstrate that one can preserve utility for a variety of 
statistical applications while achieving a formal definition of privacy.  
The algorithm we study is a simple random projection by a matrix of 
independent Gaussian random variables that compresses the number of 
records in the database. Our goal is to
preserve the privacy of every individual in the database, even if the
number of records in the database is very large.  
In particular, we
show how this randomized procedure can achieve a form of ``differential
privacy''~\cite{Dwork:06}, while at the same
time showing that the compressed data can be used for Principal
Component Analysis (PCA) and other operations that rely on the
accuracy of the empirical covariance matrix computed via the compressed data,
compared to its population or the uncompressed correspondents. Toward
this goal, we also study ``distributional privacy'' 
which is more natural for many statistical inference tasks.

More specifically, the data are represented as a $n\times p$ matrix $X$.
Each of the $p$ columns is an attribute, and each of the $n$ rows is
the vector of attributes for an individual record.  The data are
compressed by a random linear transformation $X \mapsto \X \;\equiv\;
\Phi X$, where $\Phi$ is a random $m\times n$ matrix with $m \ll
n$. It is also natural to consider a random affine transformation $X
\mapsto \X \;\equiv\; \Phi X + \Delta$, where $\Delta$ is a random
$m\times p$ matrix, as considered in~\cite{ZLW07} for privacy
analysis, the latter of which is beyond the scope of this paper and
intended as future work.  Such transformations have been called
``matrix masking'' in the privacy literature~\citep{duncan:91}.  The
entries of $\Phi$ are taken to be independent Gaussian random
variables, but other distributions are possible.  The resulting compressed
data can then be made available for statistical analyses; that is, we
think of $\X$ as ``public,'' while $\Phi$ and $\Delta$ are private and
only needed at the time of compression.  However, even if $\Phi$ were
revealed, recovering $X$ from $\X$ requires solving a highly
underdetermined linear system and comes with information theoretic
privacy guarantees, as demonstrated in~\cite{ZLW07}. 



Informally, differential privacy \cite{Dwork:06}~limits the increase
in the information that can be learned when any single entry is
changed in the database.  This limit implies \cite{MT07} that allowing
one's data to be included in the database is in some sense
incentive-compatible. Differential privacy imposes a compelling and
clear requirement, that when running a privacy-preserving algorithm on
two neighboring databases that differ in only one entry, the
probability of any possible outcome of the algorithm should be nearly
(multiplicatively) equal.  Many existing results in differential
privacy use additive output perturbations by adding a small amount of
random noise to the released information according to the sensitivity
of the {\em query} function $f$ on data $X$.  
In this work, we focus on a class $\mathcal{F}$ of Lipschitz functions that are
bounded, up to a constant $L$, by the differences between two covariance
matrices, (for example, for $\Sigma = \frac{X^T X}{n}$ and its 
compressed realization $\Sigma' = \frac{X^T \Phi^T \Phi X}{m}$ given $\Phi$),
\begin{equation}
\label{eq::lip-func}
{\cal F}(L)  = \Biggl\{ f:\ |f(A) - f(D)| 
\leq L \norm{A - D}  \Biggr\},
\end{equation}
where $A, D$ are positive definite matrices and $\norm{\cdot}$ is understood 
to be any matrix norm (for example, PCA depends on $\fnorm{\Sigma - \Sigma'}$).
Hence we focus on releasing a  
multiplicative form of perturbation of the input data, such that for a 
particular type of functions as in~\eqref{eq::lip-func}, 
we achieve both utility and privacy.
Due to the space limits, we only explore PCA in this paper.

We emphasize that although one could potentially release a version of the 
covariance matrix to preserve data privacy while performing PCA and functions 
as in~\eqref{eq::lip-func}, releasing the compressed data $\Phi X$ is 
more informative than releasing the perturbed covariance matrix 
(or other summaries) alone.
For example, Zhou et al.~\cite{ZLW07} demonstrated the
utility of this random linear transformation by analyzing the asymptotic
properties of a statistical estimator under random projection in the
high dimensional setting for $n \ll p$.
They showed that the relevant linear predictors can be learned from the 
compressed data almost as well as they could be from the original 
uncompressed data. Moreover, the actual predictions based on new 
examples are almost as accurate as they would be had the original data 
been made available.  
Finally, it is possible to release the compressed data plus some other 
features of the data to yield more information, although this is beyond the 
scope of the current paper. 
We note that in order to guarantee differential privacy,
$p < n$ is required.

In the context of guarding privacy over a set of databases 
$\Sset_n = \{X_1, X_2, \ldots\}$, where 
$\Sigma_j = X_j^T X_j/n,  \forall X_j$.
we introduce an additional parameter in our privacy definition,
$\Delta_{\max}(\Sset_n)$, which is an upper bound on pairwise distances 
between any two databases $X_1, X_2 \in \Sset_n$ (differing in any number of
rows), according to a certain distance measure.
In some sense, this parametrized approach of 
tuning the magnitude of the distance measure $\Delta_{\max}(\Sset_n)$ 
is the key idea we elaborate in Section~\ref{sec:model}.

Toward these goals, we develop key ideas in Section~\ref{sec:diff}, 
that include measure space truncation and renormalization for each measure 
$P_{\Sigma_j}, \forall j$ with Law $\Law{\cdot|X_j} \sim N(0, \Sigma_j)$;
these ideas are essential in order to guarantee differential privacy, 
which requires that even for very rare events, 
$\abs{\ln{P_{\Sigma_i}(\event)/P_{\Sigma_j}(\event)}}$ remains small 
$\forall i, j$.
We show that such rare events, when they happen not to be useful for the 
utilities that we explore, can be cut out entirely from the output space by 
simply discarding such outputs and regenerating a new $\X$. In this way, we 
provide a differential privacy guarantee by avoiding the comparisons 
made on these rare events. 
We conjecture that this is a common phenomenon rather than being specific 
to our analysis alone. In some sense, this observation is the
inspiration for our
\emph{distributional} privacy definition: over a large number $n$ of 
elements drawn from $\D$, the entire ocean of elements, 
the tail events are even more rare by the Law of Large Numbers, 
and hence we can safely truncate events whose measure $\prob{\event}$ 
decreases as $n$ increases. 

Related work is summarized in Section~\ref{section:related}.
Section~\ref{sec:definitions} formalizes privacy definitions.
Section~\ref{sec:model} gives more detail of our probability model
and summarizes our results on privacy and PCA (with proof in 
Section~\ref{sec:PCA}).
All technical proofs appear in the Appendix.

\subsection{Related Work}
\label{section:related}
Research on privacy in statistical data analysis has a long history,
going back at least to \cite{Dalenius:77b}. We refer to
\cite{duncan:91} for discussion and further pointers into this
literature; recent work includes \cite{Sanil:04}.  Recent approaches
to privacy include data swapping \cite{fienberg2004dsv}, $k$-anonymity
\cite{sweeney2002kam}, and cryptographic approaches (for instance,
\cite{pinkas2002ctp, feigen:06}).  Much of the work on data
perturbation for privacy (for example,
\cite{evfimievski2004ppm,kim2003mnm,warner1965rrs}) focuses on
additive or multiplicative perturbation of individual records, which
may not preserve similarities or other relationships within the
database. Prior to~\cite{ZLW07}, in~\cite{Agrawal:01},
an information-theoretic quantification of privacy was proposed.

  A body of recent work (for
example,~\cite{DN03,DN04,Dwork:06,DMT07,NRS07,MT07}) explores
the tradeoffs between privacy and utility while developing the
definitions and theory of {\em differential privacy}.  The two main
techniques used to achieve differential privacy to date have been
additive perturbation of individual database queries by Laplace noise
and the ``exponential mechanism''~\cite{MT07}.  
In contrast, we provide a polynomial time non-interactive algorithm for 
guaranteeing differential privacy.
Our goal is to show that, despite the general difficulty
of achieving the differential privacy guarantee, it is possible to do
so with an efficient algorithm for a specific class of functions.


The work of \cite{Liu:06} and \cite{ZLW07}, like the work presented
here, both consider low rank random linear transformations of the data
$X$, and discuss privacy and utility.  Liu et al.~\cite{Liu:06} argue
heuristically that random projection should preserve utility for data
mining procedures that exploit correlations or pairwise distances in
the data.  Their privacy analysis is restricted to observing that
recovering $X$ from $\Phi X$ requires solving an under-determined
linear system.  Zhou et al.~\cite{ZLW07} provide information-theoretic
privacy guarantees, showing that the information rate $\frac{I(X;
  \X)}{np} \to 0$ as $n \to \infty$.  Their work casts privacy in
terms of the rate of information communicated about $X$ through $\X$,
maximizing over all distributions on $X$.
Hence their analysis provides 
privacy guarantees in an average sense, whereas in this work we
prove differential privacy-style guarantees that aim to apply to every 
participant in the database semantically.
 
\section{Definitions and preliminaries}
\label{sec:definitions}
For a database $D$, let $A$ be a database access mechanism.  We
present non-interactive database privacy mechanisms, meaning that
$A(D)$ induces a distribution over sanitized output databases $\D'$.
We first recall the standard differential privacy definition from
Dwork~\cite{Dwork:06}.
\begin{definition}{\textnormal{(\sc{$\alpha$-Differential Privacy})}~\cite{Dwork:06}}
A randomized function $A$ gives $\alpha$-differential privacy if for all data sets
$D_1$ and $D_2$ differing on at most one element, and all 
$S \subseteq {\rm Range}(A)$,
$\prob{A(D_1) \in S} \leq e^{\alpha} \prob{A(D_2) \in S}.$
\end{definition}
We now formalize our notation.

{\noindent {\bf Notation: }} 
Let $\D$ be a collection of all records (potentially coming
from some underlying distribution) and $\sigma(\D)$ represent the entire 
set of input databases with elements drawn from $\D$.
Let $\Sset_n = \{X_1, X_2, \ldots\} \subset \sigma(\D)$, where 
$X_i \in \sigma(\D), \forall i$, denote a set of databases, each with $n$ 
elements drawn from $\D$. Although differential privacy is defined with 
respect to all $D, E \in \sigma(D)$, we constrain the definition of 
distributional privacy to the scope of $\Sset_n$, which becomes clear in 
Definition~\ref{def:dist-privacy}. 
We let $\cal D'$ be the entire set of possible output databases.
\begin{definition}
A privacy algorithm $A$ takes an input database $D \in \sigma(\cal D)$ and outputs
a probability measure $P_D$ on ${\cal D'}$, where $\D'$ is allowed to be different
from $\sigma(\D)$.
Let ${\cal P}$ denote all probability measures on ${\cal D'}$.
Then a privacy algorithm is a map
$A: \sigma(\cal D) \to {\cal P}$
where $A(D)= P_D, \forall D \in \sigma(\D)$.
\end{definition}

We now define differential privacy for continuous output. 
We introduce an additional parameter $\delta$ 
which measures how different two databases are according to $V$ below.
\begin{definition}
\label{def:local-density}
Let $V(D, E)$ be the distance between $D$ and $E$ according to a certain
metric, which is related to the utility we aim to provide. 
Let $d(D, E)$ denote the number of rows in which $D$ and $E$ differ.
$\delta$-constrained $\alpha$-Differential ($(\alpha, \delta)$-Differential 
Privacy)
requires the following condition,
\begin{equation}\label{eq::diffp}
\sup_{D, E: {d(D, E) = 1, V(D, E) \leq \delta}} 
\Delta(P_{D},P_{E}) \leq e^{\alpha},
\end{equation}
where 
$\Delta(P,Q) = {\rm ess}\sup_{D\in {\cal D'}}
\frac{dP}{dQ} (D)$ denotes the essential 
supremum over $\cal D'$ for the Radon-Nikodym derivative $dP/dQ$.
\end{definition}
Let $\Sset_n = \{X_1, X_2, \ldots \}$ be a set of databases
of $n$ records. Let $\Delta_{\max}(\Sset_n)$ bound the pairwise distance 
between $X_i, X_j \in \Sset_n, \forall i, j$.
We now introduce a notion of distributional privacy, 
that is similar in spirit to that in~\cite{BLR08}.
\begin{definition}
{\textnormal{\sc{(Distributional Privacy for Continuous Outcome)}}}
\label{def:dist-privacy}
An algorithm $A$ satisfies \emph{$(\alpha, \delta)$-distributional privacy}
on $\Sset_n$, for which a global parameter $\Delta_{\max}(\Sset_n)$ 
is specified, if for any two databases $X_1, X_2 \in \Sset_n$ such that 
each consists of $n$ elements drawn from $\D$, 
where $X_1 \cap X_2$ may not be empty, and for all sanitized outputs 
$\X \in \D'$, 
\be
\label{eq::distrib-comp}
 f_{X_1}(\X) \leq e^\alpha f_{X_2}(\X), \; \; \; 
\forall X_1, X_2 \text { s.t. } V(X_1, X_2)  \leq \delta
\ee
where $f_{X_j}(\cdot)$ is the density function for the conditional 
distribution with law $\Law{\cdot| X_j}, \forall i$ given $X_j$.
\end{definition}
Note that this composes nicely if one is considering databases that 
differ in multiple rows. In particular, randomness in $X_j$ is not directly 
exploited in the definition as we treat 
elements in $X_j \in \sigma(\D)$ as fixed data. One could assume that
they come from an underlying distribution, e.g., a multivariate
Gaussian $N(0, \Sigma^*)$, and infer the distance between $\Sigma_i$
and its population correspondent $\Sigma^*$.
We now show that distributional privacy is a stronger concept than 
differential privacy.
\begin{theorem}
\label{def:disttrib-diff}
Given $\Sset_n$, if $A$ satisfies $(\alpha, \delta)$-distributional privacy as 
in Definition~\ref{def:dist-privacy} for all $X_j \in \Sset_n$,
then $A$ satisfies $(\alpha, \delta)$-Differential Privacy
as in Definition~\ref{def:local-density} for all $X_j \in \Sset$.
\end{theorem}
\begin{proof}
For the same constraint parameter $\delta$, 
if we guarantee that \eqref{eq::distrib-comp} is satisfied, 
for all $X_i, X_j \in \Sset_n$ that differ only in a single row such that
$V(X_i, X_j) \leq \delta$, we have shown the $\alpha$-differential 
privacy on $\Sset_n$; clearly, this type of guarantee is necessary in 
order to guarantee $\alpha$-distributional privacy over all 
$X_i, X_j \in \Sset_n$ that satisfy the $\delta$ constraint. 
\end{proof}

\section{Probability model and summary of results}
\label{sec:model}
Let $(X_i)$ represent the matrix corresponding to $X_i \in \Sset_n$. 
By default, we use $(X_i)_j \in \R^p, \forall j =1, \ldots, n$,
and $(X^T_i)_j \in \R^n, \forall j =1, \ldots p$ 
to denote row vectors and column vectors of matrix $(X_i)$ respectively.
Throughout this paper, we assume that given any $X_i \in \Sset_n$,
columns are normalized,
\be
\label{eq::normalize}
\twonorm{(X_i^T)_j}^2 = n, \forall j=1, \ldots, p, \forall X_i \in \Sset_n
\ee
which can be taken as the first step of our sanitization scheme.
Given $X_j$, $\Phi_{m \times n}$ induces a distribution over all 
$m \times p$ matrices in $\R^{m \times p}$ via $\X = \Phi X_j$, where
$\Phi_{ij} \sim N(0, 1/n), \forall i, j$.
Let $\mathcal{L}(\cdot|X_j)$ denote the conditional distribution given 
$X_j$ and $P_{\Sigma_j}$ denote its probability measure,
where $\Sigma_j = {X_j^T X_j/n}, \forall X_j \in \Sset_n$.
Hence $\X = (x_1, \ldots, x_m)^T$ is a Gaussian Ensemble composed of 
$m$ i.i.d. random vectors with
$\Law{x_i | X_j} \sim N(0, \Sigma_j),  \forall i =1, \ldots, m.$

Given a set of databases $\Sset_n = \{X_1, X_2, \ldots\}$, 	
we do assume there is a true parameter $\Sigma^*$ such that 
$\Sigma_1, \Sigma_2, \ldots$, where $\Sigma_j = X_j^T X_j/n$,
are just a sequence of empirical parameters 
computed from databases $X_1, X_2 \ldots \in \Sset_n$. Define
\begin{eqnarray}
\label{eq::global}
\Delta_{\max}(\Sset_n) :=
2 \sup_{X_j \in \Sset_n} \max_{\ell,k} 
\abs{\Sigma_j(\ell, k) - \Sigma^*(\ell, k)}. 
\end{eqnarray}
Although we do not suppose we know $\Sigma^*$, we do compute 
$\Sigma_i, \forall i$. 
Thus $\Delta_{\max}(\Sset_n)$ provides an upper  
bound on the perturbations between any two databases $X_i, X_j \in \Sset_n$:
\begin{eqnarray}
\label{eq::global-2}
\max_{\ell,k} \abs{\Sigma_i(\ell, k) - \Sigma_j(\ell, k)} \leq 
\Delta_{\max}(\Sset_n).
\end{eqnarray}
We now relate two other parameters that measure pairwise distances
between elements in $\Sset_n$ to $\Delta_{\max}(\Sset_n)$.
For a symmetric matrix $M$, 
$\lambda_{\min}(M)$, $\lambda_{\max}(M) = \twonorm{M}$ are the smallest and 
largest eigenvalues respectively and the Frobenius norm is given by 
$\fnorm{M} = \sqrt{\sum_i \sum_j M_{ij}^2}$.
\begin{proposition}
\label{prop:variation-measure}
Subject to normalization as in~\eqref{eq::normalize}, w.l.o.g., for 
any two databases $X_1, X_j$, let $\Delta = \Sigma_1 - \Sigma_j$ and 
$\Gamma = \Sigma_j^{-1} -\Sigma_1^{-1} = 
\Sigma_j^{-1}(\Sigma_1 - \Sigma_j)\Sigma_1^{-1} 
= \Sigma_j^{-1} \Delta \Sigma_1^{-1}$.
Suppose $\max_{\ell,k} \abs{(\Sigma_1 - \Sigma_j)_{\ell k}}
\leq \Delta_{\max}(\Sset_n), \forall j$ then
\begin{eqnarray}
\label{eq::Gamma-fnorm}
\fnorm{\Delta} & \leq & p \Delta_{\max}(\Sset_n) \; \; \text{ and } \\
\fnorm{\Gamma} & \leq & 
\frac{ \fnorm{\Delta}}
{\lambda_{\min}(\Sigma_1){\lambda_{\min}(\Sigma_j)}}. 
\end{eqnarray}
\end{proposition}

Suppose we choose a reference point $\Sigma_1$
which can be thought of as an approximation to
the true value $\Sigma^*$.

\noindent{{\bf Assumption~$1$:}}
Let $\lambda_{\min}(\Sigma_1^{-1}) = \inv{\lambda_{\max}(\Sigma_1)} 
\geq C_{\min}$ for some 
constant $C_{\min} > 0$. 
Suppose $\twonorm{\Gamma} = o(1)$ and $\twonorm{\Delta} = o(1)$.

Assumption~$1$ is crucial in the sense that it guarantees that all 
matrices in $\Sset_n$ stay away from being singular
(see Lemma~\ref{lemma:banach}). We are now ready to state the 
first main result. 
Proof of the theorem appears in Section~\ref{sec:append-thm-special}.
\begin{theorem}
\label{thm:special-case}
Suppose Assumption~$1$ holds. 
Assuming that  $\twonorm{\Sigma_1}, \lambda_{\min}(\Sigma_1)$ and
$\lambda_{\min}(\Sigma_i), \forall X_i \in \Sset_n$ 
are all in the same order, and $m \geq \Omega(\ln 2np)$. 
Consider the worst case realization when 
$\fnorm{\Delta} = \Theta(p \Delta_{\max}(\Sset_n))$, 
where $\Delta_{\max} < 1$. 

In order to guard (distributional) privacy for all $X_i \in \Sset_n$
in the sense of Definition~\ref{def:dist-privacy}, it is sufficient if
\begin{equation}
\label{eq::sufficient-condition}
\Delta_{\max}(\Sset_n) = o\left(1/(p^2 \sqrt{m \ln 2np})\right).
\end{equation}
\end{theorem}
The following lemma is a standard result on existence conditions for 
$\Sigma_j^{-1}$ given $\Sigma_1^{-1}$. It also shows that 
all eigenvalue conditions in Theorem~\ref{thm:special-case} indeed hold
given Assumption~$1$.
\begin{lemma}
\label{lemma:banach}
Let $\lambda_{\min}(\Sigma_1) > 0$.
Let $\Delta = \Sigma_1 - \Sigma_j$
and $\twonorm{\Delta} < \lambda_{\min}(\Sigma_1)$.
Then 
$\lambda_{\min}(\Sigma_j) \geq \lambda_{\min}(\Sigma_1) - \twonorm{\Delta}.$
\end{lemma}

Next we use the result by Zwald and Blanchard for PCA as an instance 
from~\eqref{eq::lip-func} to illustrate the tradeoff between parameters. 
Proof of Theorem~\ref{thm:PCA} appears in Section~\ref{sec:PCA}.
\begin{proposition}\textnormal{(\cite{ZB05})}
\label{pro:ZB05}
Let $A$ be a symmetric positive Hilbert-Schmidt operator of Hilbert space 
$\mathcal{H}$ with simple nonzero eigenvalues $\lambda_1 > \lambda_2 > \ldots$.
Let $D>0$ be an integer such that $\lambda_D > 0$ and 
$\delta_D = \inv{2}(\lambda_D - \lambda_{D+1})$. Let $B \in HS(\mathcal{H})$
be another symmetric operator such that $\fnorm{B} \leq \delta_D/2$ and 
$A + B$ is still a positive operator. Let $P^D(A)$ (resp. $P^D(A + B)$) denote
the orthogonal projector onto the subspace spanned by the first $D$ eigenvectors 
$A$ (resp. $(A+B)$). Then these satisfy
\be
\label{eq:ZB}
\fnorm{P^D(A) - P^D(A + B)} \leq {\fnorm{B}}/{\delta_D}.
\ee
\end{proposition}

Subject to measure truncation of at most $1/n^2$ in each 
$P_{\Sigma_j}, \forall X \in \Sset_n$, as we show in Section~\ref{sec:diff},
we have,
\begin{theorem}
\label{thm:PCA}
Suppose Assumption~$1$ holds.
If we allow $\Delta_{\max}(\Sset_n) = O(\sqrt{{\log p}/{n}})$,
then we essentially perform PCA on the compressed sample 
covariance matrix $\X^T \X/m$ effectively in the sense of 
Proposition~\ref{pro:ZB05}: that is, in the form of~\eqref{eq:ZB} with 
$A = \frac{X^T X}{n}$ and  $B = \frac{\X^T \X}{n} - A$, 
where $\fnorm{B} = o(1)$ for $m = \Omega(p^2 \ln 2np)$. 
On the other hand, the databases in $\Sset_n$ are private in the sense of 
Definition~\ref{def:dist-privacy}, 
so long as $p^2 = O\left({\sqrt{n/m}}/{\log n}\right)$.
Hence in the worst case, we require 
$$p = o\left({n^{1/6}}/{\sqrt{\ln 2np}}\right).$$
\end{theorem}
As a special case, we look at the following example.
\begin{example}
\label{example:binary}
Let $X_1 = \{\vecx_1, \ldots, \vecx_n\}^T$ be a matrix of 
$\{-1, 1\}^{n \times p}$. 
A neighboring matrix $X_2$ is any matrix obtained via changing 
the signs on $\tau p$ bits, where $0 \leq \tau < 1$, on any $\vecx_i$.
\end{example}

\begin{corollary}
For the Example~\ref{example:binary}, it suffices 
if $p = o({n}/{\log n})^{1/4},$
in order to conduct PCA on compressed data, 
(subject to measure truncation of at most $1/n^2$ in each 
$P_{\Sigma_j}, \forall X \in \Sset_n$,) 
effectively in the sense of Proposition~\ref{pro:ZB05},
while preserve the $\alpha$-differential privacy for $\alpha = o(1)$.
\end{corollary}

\section{Distributional privacy with bounded $\Delta_{\max}(\Sset_n)$}
\label{sec:diff}
In this section, we show how we can {\it modify} the output events 
$\X$ to effectively hide some large-tail events.
We make it clear how these tail events are connected to a particular 
type of utility. 
Given $X_i$, let $\X = \Phi X_i = (x_1, \ldots, x_m)^T$.
Let $f_{\Sigma_i}(x_j) = 
\expf{-\inv{2} x_j^T \Sigma_i^{-1} x_j}/{|\Sigma_i|^{1/2} (2\pi)^{p/2}}$
be the density for Gaussian distribution $N(0, \Sigma_i)$.
Before modification, the density function $f_{\Sigma_i}(\X)$ is
\begin{equation}
\label{eq::product-function}
f_{\Sigma_i}(\X) = \prod_{j=1}^m f_{\Sigma_i}(x_j).
\end{equation}
We focus on defining two procedures that lead to both 
distributional and differential types of privacy. 
Indeed, the proof of Theorem~\ref{thm:main} applies to both, as 
the distance metric $V(X_1, X_2)$ does not specify how many rows 
$X_1$ and $X_2$ differ in. 
We use $\Delta_{\max}$ as a shorthand for $\Delta_{\max}(\Sset_n)$
when it is clear. 
\begin{procedure}
\textnormal{(\sc{Truncation of the Tail for Random Vectors in $\R^p$ })}
\label{proc:multi-truncate}
We require $\Phi$ to be an independent random draw each time we generate 
a $\X$ for compression (or when we apply it to the 
same dataset for handling a truncation event).
W.l.o.g, we choose $\Sigma_1$ to be a reference point.
Now we only examine output databases $\X \in \R^{m \times p}$ such that
for $C = \sqrt{2(C_1 + C_2)}$, where $C_1 \approx 2.5$ and 
$C_2 \approx 7.7$,
\begin{equation}
\label{eq::multivariate} 
\max_{j, k}\abs{({\X^T \X }/{m})_{jk} - \Sigma_1(j,k)}
\leq  
C \sqrt{{\ln 2np}/{m}} + \Delta_{\max},
\end{equation}
where $\Delta_{\max}(\Sset_n) = O\left(\sqrt{\log n/n}\right)$.
Algorithmically, one can imagine that for an input $X$, 
each time we see an output $\X = \Phi X$ that does not satisfy our need in 
the sense of~\eqref{eq::multivariate}, we throw the output database $\X$ away, 
and generate a new random draw $\Phi'$ to calculate $\Phi' X$ and repeat until
$\Phi' X$ indeed satisfies~\eqref{eq::multivariate}. We also
note that the adversary neither sees the databases we throw away nor finds
out that we did so.
\end{procedure}

Given $X_i \in \Sset_n$, let $\mathbb{P}_{\Sigma_i}$ be the 
probability measure over random outcomes of $\Phi X_i$.
Upon truncation, 
\begin{procedure}\textnormal{(\sc{Renormalization})}
\label{proc:multi-renorm}
We set $f'_{\Sigma_i}(\X) = 0$ for all $\X \in  \R^{m \times p}$ 
belonging to set $E$, where $E = $
\begin{eqnarray}
\label{eq::cutoff-measure} 
\left\{\X: \max_{j, k}
\small{\abs{ \left(\frac{\X^T \X }{m}\right)_{jk} - \Sigma_1(j,k)}
> C \sqrt{\frac{\ln 2np}{m}} + \Delta_{\max} }\right \},
\end{eqnarray}
corresponds to the bad events that we truncate from the outcome in 
Procedure~\ref{proc:multi-truncate};
We then renormalize the density as in~\eqref{eq::product-function} on the 
remaining $\X$ that satisfies~\eqref{eq::multivariate} to obtain:
\be
\label{eq::renorm}
f'_{\Sigma_i}(\X) = 
\frac{f_{\Sigma_i}(\X)}{1 - \probb{\Sigma_i}{{E}}}.
\ee
\end{procedure}
\begin{remark}
\label{remark:ratio}
Hence $\frac{f'_{\Sigma_1}(\X)}{f'_{\Sigma_2}(\X)} 
=  \frac{f_{\Sigma_1}(\X)(1 - \probb{\Sigma_2}{{E}})}
{f_{\Sigma_2}(\X)(1 - \probb{\Sigma_1}{E})},$
which changes
$\alpha(m, \delta)$ that we bounded below based on original density
prior to truncation of $E$ by a constant in the order of 
$\ln (1+ \eps) = O(\eps)$, where $\eps = O(1/n^2)$.
Hence we safely ignore this normalization issue given it only changes
$\alpha(m, \delta)$ by $O(1/n^2)$.
\end{remark}

The following lemma bounds the probability on the events
that we truncate in Procedure~\ref{proc:multi-truncate}.
Proof of Lemma~\ref{multi-measure-truncation} appears in
Section~\ref{sec:append-measure-truncations}.
\begin{lemma}
\label{lemma:multi-measure-truncation}
According to any individual probability measure $\mathbb{P}_{\Sigma_i}$ 
which corresponds to the sample space for outcomes of $\Phi X_i$, 
suppose that the columns of $(X_i)$ 
have been normalized to have 
$\twonorm{(X_i^T)_j}^2 = n, \forall i, j=1, \ldots, p$
and $m \geq 2(C_1 + C_2) \ln 2 n p,$  then
for $E$ as defined in~\eqref{eq::cutoff-measure}, 
$\probb{\Sigma_i}{E} \leq \inv{n^2}$.
\end{lemma}

As hinted after Definition~\ref{def:dist-privacy} 
regarding distributional privacy, we can think of the input data as 
coming from a distribution, such that $\Delta_{\max}(\Sset_n)$ 
in~\eqref{eq::global} can be derived with a typical large deviation 
bound between the sample and population covariances. 
For example, for multivariate Gaussian,
\begin{lemma}\textnormal{(\cite{RBLZ07})}
Suppose $(X_i)_j \sim N(0, \Sigma^*), \forall j =1, \ldots, n$ 
for all $X_i \in \Sset_n$, then
$\Delta_{\max}(\Sset_n) 
= O_P\left(\sqrt{{\log p}/{n}}\right).$
\end{lemma}

We now state the main result of this section.
Proof of Theorem~\ref{thm:main} appears in Section~\ref{sec:append-thm-main}.
\begin{theorem}
\label{thm:main}
Under Assumption~$1$, let $m$ and $\twonorm{(X_i^T)_j}, \forall i, j$ 
satisfy conditions in Lemma~\ref{lemma:multi-measure-truncation}.
By truncating a subset of measure at most $1/n^2$ from each
$\mathbb{P}_{\Sigma_i}$, in the sense of Procedure~\ref{proc:multi-truncate}
and renormalizing the density functions according to 
Procedure~\ref{proc:multi-renorm}, we have
\begin{eqnarray}
\label{eq::main}
\lefteqn{\alpha(m, \delta) \leq 
\frac{m p\fnorm{\Delta}}{2 \lambda_{\min}(\Sigma_i) \lambda_{\min}(\Sigma_1)}
\cdot} \\  \nonumber
& &
\left(C \sqrt{\frac{\ln 2np}{m}} + \Delta_{\max}
+ \frac{2\fnorm{\Delta} \twonorm{\Sigma_1}^2}
{p \lambda_{\min}(\Sigma_i) \lambda_{\min}(\Sigma_1)}\right) + o(1) 
\end{eqnarray}
when comparing all $X_i \in \Sset_n$ with $X_1$, where 
$\fnorm{\Gamma}$ is bounded as as in~\eqref{eq::Gamma-fnorm} 
for $i = 2$.
\end{theorem}
\begin{remark}
While the theorem only states results for comparing 
$\frac{f_{\Sigma_1}(\X)}{f_{\Sigma_i}(\X)}$, we note 
$\forall X_k, X_j \in \Sset_n$,
\begin{eqnarray*}
\abs{\ln \frac{f_{\Sigma_k}(\cdot)}{f_{\Sigma_j}(\cdot)}}
=
\abs{\ln \frac{f_{\Sigma_k}(\cdot)}{f_{\Sigma_1}(\cdot)} \cdot 
\frac{f_{\Sigma_1}(\cdot)}{f_{\Sigma_j}(\cdot)}}
 \leq  
\abs{\ln \frac{f_{\Sigma_1}(\cdot)}{f_{\Sigma_k}(\cdot)}} +
\abs{\ln \frac{f_{\Sigma_1}(\cdot)}{f_{\Sigma_j}(\cdot)}},
\end{eqnarray*}
which is simply a sum of terms as bounded as in~\eqref{eq::main}.
\end{remark}

\section{Proof of Theorem~\ref{thm:PCA}}
\label{sec:PCA}
Combining the following theorem, which illustrates the tradeoff between
the parameters $n, p$ and $m$ for PCA, with Theorem~\ref{thm:special-case}, 
we obtain Theorem~\ref{thm:PCA}.
\begin{theorem}
For a database $X \in \Sset_n$, let $A, A + B$ be the original and 
compressed sample covariance matrices respectively:
$A = \frac{X^T X}{n}$ and  $B = \frac{\X^T \X}{m} -  \frac{X^T X}{n}$,
where $\X$ is generated via Procedure~\ref{proc:multi-truncate}.
By requiring that $m = \Omega(p^2 \ln 2np)$, we can achieve meaningful 
bounds in the form of~\eqref{eq:ZB}.

\end{theorem}
\begin{proof}
We know that $A$ and $A+B$ are both positive definite, and $B$ is symmetric.
We first obtain a bound on 
$\fnorm{B} = \sqrt{\sum_{i=1}^p \sum_{j=1}^p B_{ij}^2} \leq p \tau ,$
where  
\begin{eqnarray*}
\tau & := & \max_{jk} B_{jk} =   
\max_{jk} \abs{({\X^T \X }/{m})_{jk} - A_{jk}} \\
& \leq & 
\max_{jk} 
\abs{({\X^T \X }/{m})_{jk} - \Sigma_1(j,k)}
+ \abs{\Sigma_1(j,k) - A_{jk}} \\
& \leq &
C \sqrt{{\ln 2np}/{m}} + 2 \Delta_{\max}(\Sset_n),
\end{eqnarray*}
by~\eqref{eq::multivariate},~\eqref{eq::global-2}, and the triangle inequality, for $\X = \Phi X$.
The theorem follows by Proposition~\ref{pro:ZB05} given that
$\fnorm{B} = o(1)$ for $m = \Omega(p^2 \ln 2np).$ 
\end{proof}
\noindent{\sc{ \bf {Acknowledgments.}}}
We thank Avrim Blum and John Lafferty for helpful discussions.
KL is supported in part by an NSF Graduate Research Fellowship.
LW and SZ's research is supported by NSF grant CCF-0625879,
a Google research grant and a grant from Carnegie Mellon's Cylab.

\appendix
\section{Proof of Theorem~\ref{thm:special-case}}
\label{sec:append-thm-special}
\begin{proofof}{\textnormal{Theorem~\ref{thm:special-case}}}
First we plug  $\fnorm{\Delta} = p \Delta_{\max}$ 
in~\eqref{eq::main}, and we require each term 
in~\eqref{eq::main} to be $o(1)$; hence 
we require that $p \Delta_{\max} = o(1)$ and
\begin{eqnarray}
\label{eq::alpha-formula}
p^2 \Delta_{\max} \sqrt{m \ln 2np} = o(1)  \text{ and } \hspace{1cm}
m p^2 \Delta^2_{\max} = o(1),
\end{eqnarray}
which are all satisfied given~\eqref{eq::sufficient-condition}.
Note that~\eqref{eq::alpha-formula} implies that 
$\twonorm{\Delta} = \fnorm{\Delta} = p \Delta_{\max} = o(1)$;
hence conditions in Assumption~$1$ are satisfied.
\end{proofof}

\section{Proof of Lemma~\ref{lemma:multi-measure-truncation}}
\label{sec:append-measure-truncations}
Let us first state the following lemma.
\begin{lemma}\textnormal{(See \cite{ZLW07} for example)}
\label{lemma:adapt-RSV}
Let $x, y \in \R^n$ with $\twonorm{x}, \twonorm{y} \leq 1$. Assume
that $\Phi$ is an $m \times n$ random matrix with independent $N(0,1/n)$
entries (independent of $x, y$). Then for all $\tau > 0$
\begin{equation}
\prob{\abs{\frac{n}{m}\ip{\Phi x}{\Phi y} - \ip{x}{y}} \geq \tau} \leq 
2 \exp \left(\frac{- m \tau^2}{C_1 + C_2 \tau}\right)
\end{equation}
with 
$C_1 = \frac{4 e}{\sqrt{6\pi}} \approx 2.5044$ and 
$C_2 = \sqrt{8e} \approx 7.6885$.
\end{lemma}

\begin{proofof}{Lemma~\ref{lemma:multi-measure-truncation}}
Let $X_{i, j} \in \R^n$ denote the $j^{\rm th}, \forall j = 1, \ldots, p$,
column in a $n \times p$ matrix $\forall X_i \in \Sset_n$,
W.l.o.g., we focus on $\mathbb{P}_{\Sigma_i}$ for $i = 1, 2$.
We first note that by the triangle inequality, for $X_1, X_2 \in \Sset_n$,
\begin{eqnarray*}
\lefteqn{\text{for }\; \; \X = \Phi X_1, 
\abs{\left(\frac{\X^T \X }{m}\right)_{jk} - \Sigma_1(j,k)}} \\
& = &
\abs{\inv{m} \ip{(\Phi X_1)_j}{(\Phi X_1)_k} - \inv{n} \ip{X_{1j}}{X_{1k}}} \\
\lefteqn{\text{for }\; \; \X = \Phi X_2, \; \;
\abs{\left(\frac{\X^T \X }{m}\right)_{jk} - \Sigma_1(j,k)}} \\
& \leq &
\abs{\inv{m} \ip{(\Phi X_2)_j}{(\Phi X_2)_k} - \inv{n} \ip{X_{2j}}{X_{2k}}}
 + \max_{j, k} \abs{\Delta_{jk}},
\end{eqnarray*}
where $\max_{j, k} \abs{\Delta_{jk}} \leq \Delta_{\max}(\Sset_n)$ 
by definition.

For each $\mathbb{P}_{\Sigma_i}$, we let $\event$ represents union of the 
following events, where $\tau =\frac{2(C_1 + C_2)\ln 2 n p}{m}$,
$\exists j, k\in [1, \ldots, p ],$ s.t. 
$$\abs{\inv{m}\ip{(\Phi X_i)_j}{(\Phi X_i)_k} - \inv{n}\ip{(X_i)_j}{(X_i)_k}} 
\geq \tau.$$
It is obvious that if $\event^c$ holds, we immediately have the inequality
holds in the lemma for all $\mathbb{P}_{\Sigma_i}$.  
Thus we only need to show that 
\be
\label{eq::supp-events}
\sup_{X_i \in \D}\probb{\Sigma_i}{\event} \leq 1/n^2,
\text{ where } \Sigma_i = \frac{X_i^T X_i}{n}.
\ee
We first bound the probability of a single event counted in $\event$,
which is invariant across all  $\mathbb{P}_{\Sigma_i}$ and we thus 
do not differentiate.
Consider two column vectors $x =  \frac{X_{i}}{\sqrt{n}} , 
y= \frac{X_{j}}{\sqrt{n}} \in \R^n$ in matrix $\frac{X}{\sqrt{n}}$, 
we have $\twonorm{x} = 1, \twonorm{y} = 1$. Hence for 
$\tau \leq 1$, by Lemma~\ref{lemma:adapt-RSV},
\begin{eqnarray*}
\label{eq:ip-bound}
\lefteqn{\prob{\abs{\inv{m}\ip{\Phi X_{i}}
{\Phi X_{j}} - \inv{n}\ip{X_{i}}{X_{j}}} \geq \tau}}
& & \\ 
& = & \prob{\abs{\frac{n}{m}\ip{\Phi x}{\Phi y} - \ip{x}{y}} \geq \tau}   \\
& \leq & 
2 \exp \left(\frac{- m \tau^2}{C_1 + C_2 \tau}\right) \leq 
2 \exp\left(-\frac{m \tau^2}{C_1 + C_2}\right).
\end{eqnarray*}
We can now bound the probability that any such large-deviation event happens.
Recall that $p$ is the total number of columns of $X$, hence 
the total number of events in $\event$ is 
$\frac{p(p+1)}{2}$.
Now by taking $\tau =  \sqrt{\frac{2(C_1 + C_2)\ln 2 n p}{m}}< 1$, 
where $m \geq 2(C_1 + C_2) \ln 2 n p$, we have for all $\Sigma_i$,
\begin{eqnarray*}
\probb{\Sigma_i}{\event}
& \leq & 
\frac{p(p+1)}{2} \probb{\Sigma_i}{\abs{\inv{m}\ip{\Phi X_{i}}{\Phi X_{j}} - \inv{n}\ip{X_{i}}{X_{j}}} \geq \tau} \\
& \leq &
p(p+1) \exp \left(-\frac{m \tau^2}{C_1 + C_2}\right) < \inv{n^2}
\end{eqnarray*}
This implies~\eqref{eq::supp-events} and 
hence the lemma holds.
\end{proofof}

\section{Proof of Theorem~\ref{thm:main}} 
\label{sec:append-thm-main}
W.l.o.g., we compare $\Sigma_2$ and $\Sigma_1$.
We first focus on bound $\ln \abs{\Sigma_2} - \ln \abs{\Sigma_1}$.
The following proposition comes from existing result.
\begin{proposition}\textnormal{(\cite{ZLW08})}
\label{prop:kron}
Suppose  $\twonorm{\Gamma} = o(1)$ and $\twonorm{\Delta} = o(1)$.
Then for $\Theta_i = \Sigma_i^{-1}$, 
$$\ln \abs{\Sigma_2} - \ln \abs{\Sigma_1}
= \ln \abs{\Theta_1} - \ln \abs{\Theta_2}
= A - {\rm tr}(\Gamma \Sigma_1),$$
where
$$A =
\mvec{\Gamma}^T \cdot \left(\int^1_0 (1-v)
(\Theta_1 + v  \Gamma)^{-1} \otimes (\Theta_1 + v  \Gamma)^{-1}dv
\right) \cdot \mvec{\Gamma}.$$
\end{proposition}

The lower bound on $A$ in Lemma~\ref{lemma:kron-bounds} comes from 
existing result~\cite{ZLW08}. 
We include the proof here for showing that 
the spectrum of integral term in $A$ is lower and upper bounded by that of 
$\Sigma_1$ squared, up to some small multiplicative 
constants.
\begin{lemma}\textnormal{(\cite{ZLW08})}
\label{lemma:kron-bounds}
Let $\Theta_1 = \Sigma_1^{-1}$ and 
$\Theta_2 = \Sigma_2^{-1}$, 
and hence $\Theta_2 = \Theta_1 + \Gamma$.
Under Assumption~$1$,
$$\twonorm{\Sigma_1} \twonorm{\Gamma}
\leq 
\frac{\twonorm{\Delta}\twonorm{\Sigma_1}}
{\lambda_{\min}(\Sigma_2) \lambda_{\min}(\Sigma_1)} = o(1), $$
Then
\begin{eqnarray*}
\frac{\fnorm{\Gamma}^2\lambda_{\min}(\Sigma_1)^2}
{2\left(1 + \lambda_{\min}(\Sigma_1) 
\twonorm{\Gamma}\right)^2} 
\; \leq \;  A \; \leq  \;  
\frac{\fnorm{\Gamma}^2\twonorm{\Sigma_1}^2}{2 - o(1)}.
\end{eqnarray*}
\end{lemma}
We are now ready to prove the theorem.
\begin{proofof}
{\textnormal{Theorem~\ref{thm:main}}}
Let us consider the product measure that is defined by $(\Phi_{ij})$.
The ratio of the two original density functions (prior to normalization)
for $\X = (x_1, \ldots, x_m)^T \in \R^{m\times p}$ is:
\begin{eqnarray*}
\lefteqn{\frac{f_{\Sigma_1}(\X)}{f_{\Sigma_2}(\X)} = 
\frac{\prod_{i=1}^m f_{\Sigma_1}(x_i)}
{\prod_{i=1}^m f_{\Sigma_2}(x_i)} = } \\
&  & 
\frac{|\Sigma_2|^{m/2}}{|\Sigma_1|^{m/2}} 
\expf{\sum_{i=1}^m -\inv{2} x_i^T
\left(\Sigma_1^{-1} -\Sigma_2^{-1}\right) x_i} \\
& = &
\expf{\frac{m}{2}\ln \abs{\Sigma_2} - \frac{m}{2}\ln \abs{\Sigma_1} 
+ \frac{m}{2} {\rm tr}\left(\Gamma \left(\sum_{i=1}^m \frac{x_i x_i^T }{m} 
\right)\right)}
\end{eqnarray*}
Hence by Proposition~\ref{prop:kron} and Lemma~\ref{lemma:kron-bounds}, 
we have that 
\begin{eqnarray*}
\alpha(m, \delta) & \leq & 
\abs{\frac{m}{2}\ln \abs{\Sigma_2} - \frac{m}{2}\ln \abs{\Sigma_1} 
+ \frac{m}{2} {\rm tr}\left(\Gamma \left(\frac{\X^T \X }{m} 
\right)\right) }\\
& = & 
\abs{\frac{mA}{2} - \frac{m}{2}{\rm tr}(\Gamma \Sigma_1) 
+ \frac{m}{2} {\rm tr}\left(\Gamma \left(\frac{\X^T \X }{m}\right)\right) }\\
& = & \abs{\frac{mA}{2}} + 
\abs{\frac{m}{2} 
{\rm tr}\left(\Gamma \left(\frac{\X^T \X }{m} - \Sigma_1\right)\right)}.
\end{eqnarray*}
Hence for $\X$ that satisfies~\eqref{eq::multivariate}, ignoring 
renormalization, we have for $X_1, X_2$,
\begin{eqnarray*}
\lefteqn{\alpha(m, \delta) \leq } \\
& &  
\frac{m}{2}\norm{\mvec\Gamma}_1 
\max_{j,k}\abs{\left(\frac{\X^T \X }{m}\right)_{jk} - \Sigma_1(jk)} +
\abs{\frac{mA}{2}} \\ \nonumber
& \leq & 
\frac{mp \fnorm{\Gamma}}{2}
\left(C \sqrt{\frac{\ln 2np}{m}} + \Delta_{\max}\right) + 
\frac{m \fnorm{\Gamma}^2 \twonorm{\Sigma_1}^2}{2(1 - o(1))}
\end{eqnarray*}
where
\begin{eqnarray*}
\fnorm{\Gamma} 
& \leq & 
\twonorm{\Sigma_2^{-1}}\fnorm{\Delta} \twonorm{\Sigma_1^{-1}} = 
\frac{\fnorm{\Delta}}{\lambda_{\min}(\Sigma_2) \lambda_{\min}(\Sigma_1)}.
\end{eqnarray*}
The theorem holds given~\eqref{eq::renorm} in 
Procedure~\ref{proc:multi-renorm}, Remark~\ref{remark:ratio}
and Lemma~\ref{lemma:multi-measure-truncation}.
\end{proofof}

We now show the proof for Lemma~\ref{lemma:kron-bounds}. 
Proof of Proposition~\ref{prop:kron} appears in \cite{ZLW08}.
\begin{proofof}{Lemma~\ref{lemma:kron-bounds}}
After factoring out $\twonorm{\mvec{\Gamma}}^2 = \fnorm{\Gamma}^2$, 
$A$ becomes
\begin{eqnarray}
\nonumber
\lefteqn{\lambda_{\min}\left(\int^1_0(1-v)(\Theta_1 + v \Gamma)^{-1}
\otimes (\Theta_1 + v \Gamma )^{-1} dv\right) }\\ \nonumber
\label{eq::kron-eigen}
& \geq & 
\int_0^1 (1-v)\lambda_{\min}^2(\Theta_1 + v \Gamma)^{-1} dv  \\ \nonumber
& \geq & 
\inf_{v \in [0, 1]} \lambda_{\min}^2 (\Theta_1 + v \Gamma )^{-1} 
\int_0^1 (1-v) dv \\ \nonumber
& \geq & 
\inv{2} \inf_{v \in [0, 1]}
\lambda_{\min}^2 (\Theta_1 + v \Gamma)^{-1} = 
\inv{2} \inf_{v \in [0, 1]}
\inv{\twonorm{\Theta_1 + v \Gamma}^2} \\ \nonumber
& \geq &
\inf_{v \in [0, 1]} 
\inv{2 \left(\twonorm{\Theta_1} + v \twonorm{\Gamma}\right)^2} 
\geq \inv{2\left(\twonorm{\Theta_1} + \twonorm{\Gamma}\right)^2}\\ \nonumber
& = &
\frac{\lambda_{\min}(\Sigma_1)^2}{2\left(1 + \lambda_{\min}(\Sigma_1) 
\twonorm{\Gamma}\right)^2}
\end{eqnarray}
where \eqref{eq::kron-eigen} is due to the fact that the set of $p^2$ 
eigenvalues of $B(v) \otimes B(v)$, where 
$B(v) = (\Theta_1 + v \Gamma)^{-1}$, $\forall v \in [0, 1]$,  
is $\{\lambda_i(B(v)) \lambda_j(B(v)), \forall i, j =1, \ldots, p\}$,
which are all positive given that
$(\Theta_1 + v \Gamma) \succ 0$, hence $(\Theta_1 + v \Gamma)^{-1} \succ 0$,
$\forall v \in[0, 1]$ as shown above.
Similarly, 
\begin{eqnarray*}
\nonumber
\lefteqn{\lambda_{\max}\left(\int^1_0(1-v)(\Theta_1 + v \Gamma)^{-1}
\otimes (\Theta_1 + v \Gamma )^{-1} dv\right) } \\ \nonumber
& \leq & 
\int_0^1 (1-v)\lambda_{\max}^2(\Theta_1 + v \Gamma)^{-1} dv  \\
& \leq &  
\sup_{v \in [0, 1]} \lambda_{\max}^2 (\Theta_1 + v \Gamma )^{-1} 
\int_0^1 (1-v) dv \\ \nonumber
& \leq & 
\sup_{v \in [0, 1]} \inv{2 \lambda_{\min}^2 (\Theta_1 + v \Gamma)},
\end{eqnarray*}
where $\forall v \in [0, 1]$,
\begin{eqnarray*}
\lambda_{\min} (\Theta_1 + v \Gamma) \geq 
\lambda_{\min}(\Theta_1) - v \twonorm{\Gamma} = 
\frac{1 - v \twonorm{\Gamma} \twonorm{\Sigma_1}}
{\twonorm{\Sigma_1}} > 0
\end{eqnarray*}
where so long as $\twonorm{\Delta} = o(1)$ and 
$\twonorm{\Sigma_1}$ and $\lambda_{\min}(\Sigma_1)$ are
within constant order of each other, we have
$$
\twonorm{\Sigma_1} \twonorm{\Gamma}
\leq 
\frac{\twonorm{\Delta}\twonorm{\Sigma_1}}
{\lambda_{\min}(\Sigma_2) \lambda_{\min}(\Sigma_1)} = o(1).$$ 
Hence $\forall v \in [0, 1]$,
\begin{eqnarray*}
\inv{\lambda_{\min} (\Theta_1 + v \Gamma)}
& \leq &
 \frac{\twonorm{\Sigma_1}}{1- v \twonorm{\Sigma_1} 
\twonorm{\Gamma}} 
\leq
\frac{\twonorm{\Sigma_1}}{1- \twonorm{\Sigma_1} 
\twonorm{\Gamma}}
\end{eqnarray*}
and correspondingly
\begin{eqnarray*}
\sup_{v \in [0, 1]} \inv{2 \lambda_{\min}^2 (\Theta_1 + v \Gamma)}
& \leq & 
\frac{\twonorm{\Sigma_1}^2}{2\left(1- \twonorm{\Sigma_1} 
\twonorm{\Gamma}\right)^2}.
\end{eqnarray*}
\end{proofof}

\section{Example with binary data matrix}
\label{sec:binary}
We now show how we can achieve differential privacy in the 
setting where $X_1, X_2 \in \R^{np}$ such that they differ in a single 
row. We also define the some special case of this general setting.
We further illustrate the idea that one can not allow differential
privacy without giving certain constraints on $X_1, X_2, \ldots$.
As a corollary of Theorem~\ref{thm:main}, 
we consider the following example.
\begin{proposition}
\label{prop:binary-Delta}
For the Binary Game in Example~\ref{example:binary}, we have for 
all $\tau \leq 1$ and $x \in \R^p$,
$$\twonorm{\Delta} \leq \fnorm{\Delta} \leq \frac{2p \sqrt{\tau(1-\tau)}}{n}
\leq \frac{p}{n}.
$$
\end{proposition}
\begin{proof}
For the special case that $X_1$ and $X_2$ differ only in a single 
row after normalization, 
such that $x \in X_1$ and $y \in X_2$, we have
$\Delta  =  \Sigma_1 - \Sigma_2 =  \frac{x x^T - y y^T}{n}$.
First note $\twonorm{\Delta} \leq \fnorm{\Delta}$.
In order to bound $\fnorm{\Delta}$ let us define
\begin{eqnarray}
B  = x x^T - y y^T, \hspace{0.5cm} C = (x- y)x^T, \hspace{0.5cm} D =  y (x - y)^T,
\end{eqnarray}
where $B = C + D$ are all $p \times p$ matrices. A careful counting of non-zero
elements in $C + D$ gives 
$\twonorm{\Delta} 
\leq \fnorm{\Delta} = 2 \frac{p}{n} \sqrt{\tau(1-\tau)}
\leq \frac{p}{n}$ for $\tau \leq 1/2$. Note that when $\tau > 1/2$, the 
effect on $\fnorm{\Delta}$ is the same as flipping $1 - \tau$ bits, hence
it is maximized when $\tau = 1/2$.
\end{proof}

\begin{theorem}
In the binary game in~\ref{example:binary}, 
by truncating a subset of measure at most $1/n^2$, we have
\begin{eqnarray*} 
\alpha(m, \tau)  \leq 
\abs{\frac{m p^2 (C \sqrt{{\ln 2np}/{m}} + O({1}/{n}))} 
{n \lambda_{\min}(\Sigma_2) \lambda_{\min}(\Sigma_1)}} = o(1)
\end{eqnarray*}
for $p = o(\sqrt{n/m})$ and $m \geq \Omega(\ln 2np)$.
\end{theorem}
\begin{proof}
Note that for a binary game, 
$\Delta_{\max} = \max_{jk} \Delta_{jk} \leq \frac{2}{n}.$
As shown in Proposition~\ref{prop:binary-Delta},
$$\fnorm{\Delta} \leq \frac{p}{n}.$$ 
Plugging the above inequalities in~\eqref{eq::main}, we have
\begin{eqnarray*}
\lefteqn{ \alpha(m, \tau) \leq
\frac{m p^2}{n \lambda_{\min}(\Sigma_2)\lambda_{\min}(\Sigma_1)} \cdot} \\
& & 
\left(C\sqrt{\frac{\ln 2np}{m}} + \frac{2}{n} + 
\frac{2 \twonorm{\Sigma_1}^2}{n
{\lambda_{\min}(\Sigma_2) \lambda_{\min}(\Sigma_1)}}
\right) + o(1)
\end{eqnarray*}
where for $p = o(\sqrt{n/m})$ and for 
$m \geq \Omega(\ln 2np)$, we have 
$\alpha(m, \tau) = o(1)$.
\end{proof}

\end{document}